\documentclass{article}

\usepackage{microtype}
\usepackage{graphicx}
\usepackage{subfigure}
\usepackage{hyphenat}
\usepackage{booktabs} 

\usepackage{hyperref}

\usepackage[accepted]{icml2021}

\newcommand{\GG}[1]{}

\icmltitlerunning{Equivariant Graph Neural Networks for 3D Macromolecular Structure}
\usepackage{amsthm, makecell, float, amssymb, amsfonts, graphicx, bbm, multirow} 
\newcommand{\vvv}{\mathbf{v}}
\newcommand\norm[1]{{\left\lVert#1\right\rVert}_2} 

\usepackage{amsmath,amsfonts,bm}

\def\eqref#1{equation~\ref{#1}}

\def\1{\bm{1}}

\def\eps{{\epsilon}}

\DeclareMathAlphabet{\mathsfit}{\encodingdefault}{\sfdefault}{m}{sl}
\SetMathAlphabet{\mathsfit}{bold}{\encodingdefault}{\sfdefault}{bx}{n}

\begin{document}

\twocolumn[
\icmltitle{Equivariant Graph Neural Networks for 3D Macromolecular Structure}




\begin{icmlauthorlist}
\icmlauthor{Bowen Jing}{cs}
\icmlauthor{Stephan Eismann}{cs,ap}
\icmlauthor{Pratham N. Soni}{cs}
\icmlauthor{Ron O. Dror}{cs}
\end{icmlauthorlist}

\icmlaffiliation{cs}{Department of Computer Science, Stanford University, USA}
\icmlaffiliation{ap}{Department of Applied Physics, Stanford University, USA}

\icmlcorrespondingauthor{Bowen Jing}{bjing@cs.stanford.edu}
\icmlcorrespondingauthor{Stephan Eismann}{seismann@cs.stanford.edu}
\icmlcorrespondingauthor{Ron O. Dror}{rondror@cs.stanford.edu}

\icmlkeywords{equivariance, graph neural network, molecular structure}

\vskip 0.3in
]



\printAffiliationsAndNotice{}
\newcommand{\hide}[1]{}
\begin{abstract}

Representing and reasoning about 3D structures of macromolecules is emerging as a distinct challenge in machine learning. Here, we extend recent work on \emph{geometric vector perceptrons} and apply equivariant graph neural networks to a wide range of tasks from structural biology. Our method outperforms all reference architectures on three out of eight tasks in the ATOM3D benchmark, is tied for first on two others, and is competitive with equivariant networks using higher-order representations and spherical harmonic convolutions. In addition, we demonstrate that transfer learning can further improve performance on certain downstream tasks. Code is available at \url{https://github.com/drorlab/gvp-pytorch}.
\end{abstract}

\section{Introduction}

Learning on 3D structures of macromolecules (such as proteins) is a rapidly growing area of machine learning with promising applications but also domain-specific challenges. In particular, methods should possess an efficient and precise representation of structures with thousands of atoms and faithfully reason about their 3D geometry independent of orientation and position \citep{laine2021protein}. 

\emph{Equivariant neural networks} (ENNs), operating on point-cloud representations of structures with geometric vector/tensor features and spherical harmonic convolutions, address these challenges. These networks have shown promising results on small molecule datasets \cite{thomas2018tensor, kondor2018n, anderson2019cormorant, fuchs2020se3, batzner2021se} and recently also in the context of learning from macromolecular structure  \citep{eismann2020one,eismann2020two,townshend2020atom3d}.

\begin{figure}
    \centering
    \includegraphics[width=0.7\linewidth]{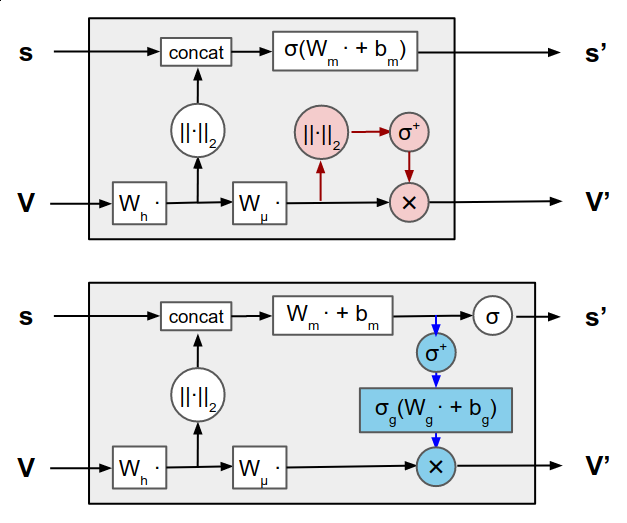}
    \caption{Schematic of the original geometric vector perceptron (GVP) as described in \citet{jing2021learning} (\textbf{top}) and the modified GVP presented in Algorithm~\ref{alg:gvp2} (\textbf{bottom}). The original vector nonlinearity (in red) has been replaced with vector gating (in blue), allowing information to propagate from the scalar channels to the vector channels. Circles denote row- or element-wise operations. The modified GVP is the core module in our equivariant GNN.}
    \label{fig:gvp}
\end{figure}

\emph{Graph neural networks} (GNNs) have enjoyed widespread use on macromolecular structure due to the natural representation of structures as graphs, with residues or atoms as nodes and edges drawn based on bonding or spatial proximity \citep{ingraham2019generative, baldassarre2020graphqa}. However, these networks generally indirectly encode the 3D geometry in terms of pairwise distances, angles, and other scalar features. Equivariant message-passing, proposed by \citet{jing2021learning} and more recently \citet{satorras2021en} and \citet{schuett2021equivariant}, seeks to instead incorporate the equivariant representations of ENNs  within the message-passing framework of GNNs. This presents an alternative, mathematically simpler approach to equivariance in lieu of the equivariant convolutions of ENNs, and also has the advantage of leveraging the relational reasoning of GNNs.

Here, we present a GNN with equivariant message-passing and demonstrate that it achieves strong results across a wide variety of tasks on macromolecular structures. Our architecture extends GVP-GNNs \cite{jing2021learning} and we evaluate its performance on ATOM3D \cite{townshend2020atom3d}, a comprehensive suite of tasks and datasets from structural biology. We also demonstrate that our architecture can leverage transfer learning---i.e., pre-training on a data-rich task to improve performance on a data-poor task. This is particularly desirable for machine learning on macromolecules, as structural data is often expensive and difficult to obtain.

\section{Equivariant Graph Neural Networks}

GVP-GNNs \citep{jing2021learning} are equivariant graph neural networks in which all node and edge embeddings are tuples $(\mathbf{s}, \mathbf{V})$ of scalar features $\mathbf{s} \in \mathbb{R}^n$ and geometric vector features $\mathbf{V} \in \mathbb{R}^{\nu \times 3}$. Message and update functions are parameterized by \emph{geometric vector perceptrons} (GVPs)---modules mapping between tuples $(\mathbf{s}, \mathbf{V})$ while preserving rotation equivariance. Here, we extend GVP-GNNs---originally designed for representation of protein structures at the amino acid residue level---to atomic-level structure representations. This enables the architecture to be applied to a larger variety of tasks, and motivates \emph{vector gating} as an adaptation for the change in input representation.

\paragraph{Vector gating} In the originally defined GVP, the vector outputs are functions of the vector inputs, but not the scalar inputs---i.e, vector features at all stages of message-passing are independent of residue identity or other scalar features. This is not a significant issue for residue-level structure graphs, in which nodes start with vector features in the form of residue orientations. However, in atomic-level structure graphs, individual atoms do not necessarily start with an orientation. We therefore introduce \textit{vector gating} to propagate information from the scalar channels into the vector channels (Algorithm~\ref{alg:gvp2}). In this operation, the scalar features $\text{s}_m$ are transformed and passed through a sigmoid activation $\sigma_g$ in order to ``gate'' the vector output $\textbf{V}'$, replacing the vector nonlinearity. Because $\text{s}_m$ is invariant and the gating is row-wise, the equivariance of $\textbf{V}'$ is unaffected.

\begin{algorithm}
\small
\caption{Geometric vector perceptron (with {\color{blue}vector gate})}\label{alg:gvp2}
\begin{algorithmic}
\STATE \textbf{Input}: Scalar and vector features $(\textbf{s},\textbf{V})$ $\in$ $\mathbb{R}^n \times \mathbb{R}^{\nu \times 3}$ .
\STATE \textbf{Output}: Scalar and vector features $(\textbf{s}',\textbf{V}')$ $\in$ $\mathbb{R}^m \times \mathbb{R}^{\mu \times 3}$ .
\STATE $h \gets \max \left(\nu, \mu\right)$ (or separately specified)

\STATE \textbf{GVP}: 
\STATE \ \ \ \ $\text{V}_h \gets \textbf{W}_h \textbf{V} \quad \in \mathbb{R}^{h \times 3} $ 
\STATE \ \ \ \ $\text{V}_\mu \gets \textbf{W}_\mu \text{V}_h \quad \in \mathbb{R}^{\mu \times 3} $ 
\STATE \ \ \ \ $\text{s}_h \gets \norm{\text{V}_h} \ (\text{row-wise}) \quad \in \mathbb{R}^{h}$ 
\STATE \ \ \ \ $\text{s}_{h+n} \gets \text{concat}\left(\text{s}_h, \textbf{s}\right) \quad \in \mathbb{R}^{h+n}$
\STATE \ \ \ \ $\text{s}_m \gets \textbf{W}_m \text{s}_{h+n} + \textbf{b}_m \quad \in \mathbb{R}^m$
\STATE \ \ \ \ $\textbf{s}' \gets \sigma(\text{s}_m) \quad \in \mathbb{R}^{m}$
\STATE \ \ \ \ $\textbf{V}' \gets {\color{blue} \sigma_g\left(\textbf{W}_g[\sigma^+(\text{s}_m)]+\textbf{b}_g\right)\odot \text{V}_\mu} \ (\text{row-wise}) \quad \in \mathbb{R}^{\mu \times 3}$ 
\STATE  \textbf{return} $(\textbf{s}',\textbf{V}')$  
\end{algorithmic}
\end{algorithm}

The modification enables the GVP to inherit the Universal Approximation Property of dense layers with respect to rotation- and reflection-\textit{equivariant} functions $F: \mathbb{R}^{\nu \times 3} \rightarrow \mathbb{R}^3$, in addition to the approximation property for \textit{invariant} functions shown by \citet{jing2021learning}.

\newtheorem*{theorem}{Theorem}
\begin{theorem}
    Let R describe an arbitrary rotation and/or reflection in $\mathbb{R}^3$ and $G_s$ be the vector outputs of a GVP defined with $n=0$, $\mu=6$, and sigmoidal $\sigma, \sigma^+$.
    For $\nu \ge 3$ let $\Omega^\nu \subset \mathbb{R}^{\nu \times 3}$ be the set of all $\mathbf{V} = \begin{bmatrix}\vvv_1, & \ldots, & \vvv_\nu\end{bmatrix}^T\in \mathbb{R}^{\nu \times 3}$ such that $\vvv_1, \vvv_2, \vvv_3$ are linearly independent and $0<||\vvv_i||_2 \le b$ for all $i$ and some finite $b>0$. Then for any continuous $F: \Omega^\nu \rightarrow \mathbb{R}^3$ such that $(F \circ R)(\mathbf{V}) = (R\circ F)(\mathbf{V})$ and for any $\epsilon>0$, there exists a form $f(\mathbf{V})=\mathbf{1}^TG_s(\mathbf{V})$ such that $|F(\mathbf{V})_i-f(\mathbf{V})_i|< 6bC\epsilon$ (with finite $C>0$) for all $i=1,2,3$ and all $\mathbf{V}\in\Omega^\nu$. 
\end{theorem}
See the Appendix for a proof. As a corollary, a GVP with nonzero $n$---that is, with scalar inputs---can approximate similarly defined functions over the full input domain $\mathbb{R}^n \times \mathbb{R}^{\nu \times 3}$, corresponding to the propagation of information from scalar to vector channels. Using these modified GVPs, we build GVP-GNN with the same formulation of message-passing and node update as in \citet{jing2021learning}.

\section{ML on 3D Macromolecular Structure}

ATOM3D \citep{townshend2020atom3d} is a collection of eight tasks and datasets for learning on atomic-level 3D molecular structure. These tasks span several classes of molecules (proteins, RNAs, small molecules, and complexes) and problem formulations (regression and classification) encountered in  structural biology (Table~\ref{tab:atom3d}). \citet{townshend2020atom3d} established reference benchmarks for convolutional, graph, and equivariant neural networks---the three architectures most broadly applicable to molecular structure---on these tasks. We compare our modified equivariant GVP-GNN to these reference architectures.

The tasks in ATOM3D span several orders of magnitude in both dataset size and structure size (a rough proxy for task complexity), and therefore range from data-rich (SMP, RES) to data-poor (MSP, LEP). However, these tasks share a common underlying representation and problem domain, presenting an opportunity for \emph{transfer learning} to improve performance on data-poor tasks. We therefore investigate if leveraging the learned parameters from the first two layers of the trained model on SMP (target $\mu$) and RES can improve over a non-pretrained model on the other tasks.

\begin{table*}[ht!]
  \caption{Tasks in ATOM3D \citep{townshend2020atom3d} encompass a broad range of structure types, problem formulations, dataset sizes, and structure sizes. PPI, RES, MSP, and LEP are classification tasks; the others are regression tasks. Further, PPI, MSP, and LEP have paired inputs, so methods for these tasks use weight-tying between two independent forward passes. Dataset size is the number of training examples, and structure size is the geometric mean of the number of nodes in training examples. In paired tasks, each pair is counted as two examples. Where necessary, we use dataloader conventions from \citep{townshend2020atom3d}.}
  \label{tab:atom3d}
  \centering
  \small
  \begin{tabular}{ccccc} 
    \toprule
    Name & Input Structure(s) & Prediction & Dataset Size & Structure Size\\
   \midrule
   \makecell{Small Molecule\\Properties (SMP)} & Small molecule & Physiochemical property & $104 \times 10^3$ & 18 \\
   \midrule
   \makecell{Protein-Protein\\Interface (PPI)} & \makecell{Two protein chains\\+ residue from each} & \makecell{Whether the residues come into contact\\when the parent chains interact} & $2.1 \times 10^6$ & 999 \\
   \midrule
   \makecell{Residue Identity\\(RES)} & \makecell{Environment around\\a masked residue} & \makecell{Masked residue identity} & $1.0 \times 10^6$ & 574\\
   \midrule
   \makecell{Mutation Stability\\Prediction (MSP)} & \makecell{Protein complex + same\\complex with mutation} & \makecell{Whether the mutation\\stabilizes the complex} & 5728 & 3236\\
   \midrule
   \makecell{Ligand Binding\\Affinity (LBA)} & \makecell{Protein-ligand complex} & \makecell{Negative log affinity ($\text{p}K_\text{d}$)} & 3507 & 382\\
   \midrule
   \makecell{Ligand Efficacy\\Prediction (LEP)} & \makecell{Protein-ligand complex in\\active+inactive conformations} & \makecell{Whether the ligand\\activates the protein} & 608 & 3001\\
   \midrule
   \makecell{Protein Structure\\Ranking (PSR)} & \makecell{Protein} & \makecell{GDT-TS relative to\\native structure} & $25.4 \times 10^3$ & 1384 \\
   \midrule
   \makecell{RNA Structure\\Ranking (RSR)} & \makecell{RNA} & \makecell{RMSD relative to\\native structure} & $12.5 \times 10^3$ & 2161\\
   \bottomrule
  \end{tabular}
\end{table*}

\begin{table*}[ht!]
  \centering
  \caption{Comparison of the GVP-GNN with the CNN, GNN, and ENN reference architectures on ATOM3D. The GVP-GNN is the best architecture on three out of eight tasks, more than any other method. 
  For each metric, the best model is in bold. SMP metrics are MAE. Metrics are labeled with $\uparrow$/$\downarrow$ if higher/lower is better, respectively. Results are mean $\pm$ S.D. over three training runs.}
  \small
    \begin{tabular}{llcccc}
    \toprule
    & & \multicolumn{3}{c}{ATOM3D Reference} \\
    \cmidrule(r){3-5}
    Task & Metric & CNN & GNN & ENN & GVP-GNN \\
    \midrule
    SMP & $\mu$ [D] $\downarrow$ & $0.754 \pm 0.009$ & $0.501 \pm 0.002$ & ${0.052 \pm 0.007}$ & $\mathbf{0.049 \pm 0.002}$ \\
    & $\eps_\text{gap}$ [eV] $\downarrow$ & $0.580 \pm 0.004$ & $0.137 \pm 0.002$ & ${0.095 \pm 0.021}$ & $\mathbf{0.065 \pm 0.001}$ \\
    & $U_0^\text{at}$ [eV] $\downarrow$ & $3.862 \pm 0.594$ & $1.424 \pm 0.211$ & $\mathbf{0.025 \pm 0.001}$ & $0.143 \pm 0.007$ \\
    \midrule
    PPI & AUROC $\uparrow$ & ${0.844 \pm 0.002}$  & $0.669 \pm 0.001$ & --- & $\mathbf{0.866 \pm 0.004}$ \\ \midrule
    RES & accuracy $\uparrow$ & ${0.451 \pm 0.002}$  & $0.082 \pm 0.002$ & $0.072 \pm 0.005$ & $\mathbf{0.527 \pm 0.003}$ \\ \midrule
    MSP & AUROC $\uparrow$ & $0.574 \pm 0.005$  & $0.621 \pm 0.009$ & $0.574 \pm 0.040$ & $\mathbf{0.680 \pm 0.015}$ \\ \midrule
    LBA & RMSE $\downarrow$ & $\mathbf{1.416 \pm 0.021}$  & $1.570 \pm 0.025$ & ${1.568 \pm 0.012}$ & $1.594 \pm 0.073$\\
    \midrule
    LEP & AUROC $\uparrow$ & ${0.589 \pm 0.020}$ & $\mathbf{0.740 \pm 0.010}$ & $0.663 \pm 0.100$ & $0.628 \pm 0.055$\\ \midrule
    PSR & mean $R_S$ $\uparrow$ & $0.431 \pm 0.013$ & $\mathbf{0.515 \pm 0.010}$ & --- & ${0.511 \pm 0.010}$ \\
    & global $R_S$ $\uparrow$ & $0.789 \pm 0.017$ & $0.755 \pm 0.004$ & --- & $\mathbf{0.845 \pm 0.008}$ \\
    \midrule
     RSR & mean $R_S$ $\uparrow$ & $\mathbf{0.264 \pm 0.046}$ & $0.234 \pm 0.006$ & --- & $0.211 \pm 0.142$ \\
    & global $R_S$ $\uparrow$ & ${0.372 \pm 0.027}$ & $\mathbf{0.512 \pm 0.049}$ & --- & $0.330 \pm 0.054$ \\
    \bottomrule
  \end{tabular}
  \label{tab:atom3d2}
\end{table*}

\begin{table*}[ht!]
    \caption{Transfer learning improves performance of GVP-GNNs on a number of ATOM3D tasks (PPI, MSP, PSR). For each task, we compare models pretrained on SMP and RES with the original models. Pretrained models that outperform the original models are shown in bold. The original model is shown in bold if neither pretrained model outperforms it. $p$-values are shown for a one-tailed $t$-test between the pretrained model and original model, and bolded if $p\le 0.05$. Metrics are labeled with $\uparrow$/$\downarrow$ if higher/lower is better. Results are mean $\pm$ S.D. over three training runs.}
    \centering
    \small
    \begin{tabular}{llccccccc}
    \toprule
    Task & Metric & No Pretraining & & SMP Pretraining & $p$ & & RES Pretraining & $p$ \\
    \midrule
    PPI & AUROC $\uparrow$ & $0.866 \pm 0.004$ & & $\mathbf{0.874 \pm 0.001}$ & \textbf{0.04} & & --- & --- \\ \midrule
    RES & accuracy $\uparrow$ & ${0.527 \pm 0.003}$ & & $\mathbf{0.531 \pm 0.001}$ & 0.07 & & --- & --- \\ \midrule
    MSP & AUROC $\uparrow$ & $0.680 \pm 0.015$ & & $\mathbf{0.711 \pm 0.007}$ & \textbf{0.03} & & $\mathbf{0.709 \pm 0.014}$ & \textbf{0.04} \\ \midrule
    LBA & RMSE $\downarrow$ & $\mathbf{1.594 \pm 0.073}$ & & $1.649 \pm 0.118$ & 0.73 & & $1.676 \pm 0.162$ & 0.76\\ \midrule
    LEP & AUROC $\uparrow$ & $\mathbf{0.628 \pm 0.055}$ & & $0.567 \pm 0.210$ & 0.67 & & $0.466 \pm 0.110$ & 0.95 \\ \midrule
    PSR & mean $R_S$ $\uparrow$ & $0.511 \pm 0.010$ & & $\mathbf{0.515 \pm 0.007}$ & 0.31 & & $\mathbf{0.515 \pm 0.008}$ & 0.35\\
        & global $R_S$ $\uparrow$ & $0.845 \pm 0.008$ & & $\mathbf{0.848 \pm 0.006}$ & 0.29 & & $\mathbf{0.862 \pm 0.006}$ & \textbf{0.02}\\ 
    \midrule
    RSR & mean $R_S$ $\uparrow$ & $\mathbf{0.211 \pm 0.142}$ & & $0.194 \pm 0.161$ & 0.55 & & $0.156 \pm 0.155$ & 0.66 \\
        & global $R_S$ $\uparrow$ & $\mathbf{0.330 \pm 0.054}$ & & $0.059 \pm 0.303$ & 0.87 & & $0.308 \pm 0.072$ & 0.65\\ 
    \bottomrule
    \end{tabular}
    \label{tab:transfer}
\end{table*}

\paragraph{Architecture}
We represent a macromolecular structure as a graph $\mathcal{G} = (\mathcal{V}, \mathcal{E})$ where each node $\mathfrak{v}_i \in \mathcal{V}$ corresponds to an atom and is featurized by a one-hot encoding of its element type. We draw edges $\mathcal{E}=\{\mathbf{e}_{j\rightarrow i}\}_{i\neq j}$ for all $i, j$ whose pairwise distance is less than 4.5 \r{A}. Each edge is featurized with a unit vector in the direction of the edge and a Gaussian RBF encoding of its length. We use hidden embeddings with 16 vector and 100 scalar channels. For all tasks, we use a GVP-GNN with five layers, followed by a mean pool (unless otherwise noted) and two dense layers. For all GVPs, we use $\sigma=\text{ReLU}$ and $\sigma^+ = \text{id}$.

We do not encode any bond or bond type information, nor do we distinguish atoms from different molecules in the same structure. This is in order to standardize the architecture across tasks as much as possible, but may handicap our method relative to those that do use this information.

In all tasks except SMP and LBA, we omit the hydrogen atoms. In the pairwise tasks PPI, LEP, and MSP, we perform independent forward passes and concatenate the outputs. In residue-specific tasks, we pick out the output embedding of the alpha carbon (PPI, RES), or mean pool over the residue (MSP). These conventions are from \citet{townshend2020atom3d}.

\section{Results}

We compare the vector-gated GVP-GNN with the reference six-layer CNN, five-layer GNN, and four-layer ENN architectures on the ATOM3D test sets (Table~\ref{tab:atom3d2}). The GVP-GNN achieves generally strong performance: it is the top method on three tasks (PPI, RES, MSP)---more than any other method---and tied for first on two (SMP, PSR). 
However, in a head-to-head comparison with the standard GNN, the GVP-GNN scores higher on only 8 out of 12 metrics, suggesting that equivariant, vector-valued representations may not be equally useful for all datasets or metrics.

The reference ENN is a Cormorant neural network---a point cloud network with spherical harmonic convolutions \citep{anderson2019cormorant}. GVP-GNN is on par with this ENN, scoring better on 4 out of 7 metrics. (The ENN was not trained on PPI, PSR, and RSR.) However, the ENN uses equivariant representations up to $L=3$---comparable to 3rd order geometric tensors---while the GVP-GNN only uses geometric \textit{vectors}, or 1st order tensors. Thus, equivariant message-passing with lower order tensors is competitive with higher-order ENNs on molecular tasks.

Compared with the GNN and ENN collectively, the GVP-GNN falls short on the tasks involving ligand-protein complexes (LBA, LEP). We postulate that this may be because our general atomic-level input features omit information that explicitly distinguishes ligand atoms from protein atoms.

\paragraph{Transfer learning}

GVP-GNNs pretrained on data-rich tasks (SMP, RES) are able to improve performance on a number of downstream tasks (Table~\ref{tab:transfer}). Specifically, the SMP-pretrained model improves performance on PPI, RES, MSP, and PSR, and the RES-pretrained model improves performance on MSP and PSR. (The RES model was not fine-tuned on PPI because the datasets are of similar size.) Many of these improvements are statistically significant at the $\alpha=0.05$ level. To the best of our knowledge, this is the first time transfer learning has been successfully demonstrated for machine learning methods operating on 3D representations of macromolecular structures.

We observe that even in cases where pretraining does not improve the performance, it still may expedite training. Figure~\ref{fig:curves} compares the learning curves for the original and SMP-pretrained models on PPI, RES, and LBA. On all three tasks, the learned weights from SMP appear to help on the new dataset, as learning is expedited by several epochs. However, while for PPI and RES this head start leads to improved or comparable performance, for LBA it increases overfitting and hurts performance.

\begin{figure}
    \centering
    \includegraphics[width=\linewidth]{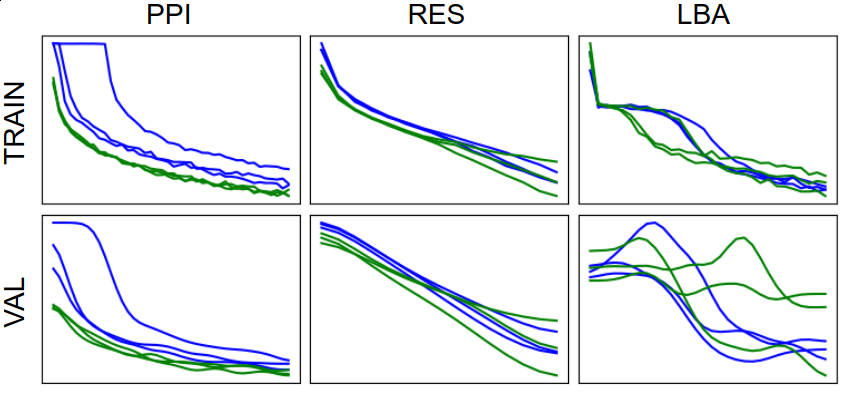}
    \caption{Learning curves (loss vs. epoch) for the original (blue) and SMP-pretrained (green) models on PPI, RES, LBA, all illustrating expedited training but with different final results. The losses have been Gaussian smoothed with $\sigma=2$ epochs.}
    \label{fig:curves}
\end{figure}

\section{Conclusion}

We have demonstrated the systematic application of equivariant graph neural networks to macromolecular structures. Our architecture extends GVP-GNN to atomic-level structures, endows it with additional universal approximation properties, and achieves strong performance across a wide range of structural biology tasks. We also show that leveraging pretrained equivariant representations can boost performance on downstream tasks. These results suggest that equivariant message-passing is a powerful and versatile paradigm for machine learning on molecules.

\newpage
{\small
\bibliography{bibliography}
\bibliographystyle{icml2021}
}
\newpage
\section*{Appendix}

\begin{theorem}
    Let R describe an arbitrary rotation and/or reflection in $\mathbb{R}^3$ and $G_s$ be the vector outputs of a GVP defined with $n=0$, $\mu=6$, and sigmoidal $\sigma, \sigma^+$.
    For $\nu \ge 3$ let $\Omega^\nu \subset \mathbb{R}^{\nu \times 3}$ be the set of all $\mathbf{V} = \begin{bmatrix}\vvv_1, & \ldots, & \vvv_\nu\end{bmatrix}^T\in \mathbb{R}^{\nu \times 3}$ such that $\vvv_1, \vvv_2, \vvv_3$ are linearly independent and $0<||\vvv_i||_2 \le b$ for all $i$ and some finite $b>0$.    Then for any continuous $F: \Omega^\nu \rightarrow \mathbb{R}^3$ such that $(F \circ R)(\mathbf{V}) = (R\circ F)(\mathbf{V})$ and for any $\epsilon>0$, there exists a form $f(\mathbf{V})=\mathbf{1}^TG_s(\mathbf{V})$ such that $|F(\mathbf{V})_i-f(\mathbf{V})_i|< 6bC\epsilon$ (with finite $C>0$) for all $i=1,2,3$ and all $\mathbf{V}\in\Omega^\nu$. 
\end{theorem}

\begin{proof}
We start by writing out the GVP $G_s$ in full as $G_s(\textbf{V}) = \sigma_g\left(\textbf{W}_g\text{s}^+ +\textbf{b}_g\right) \odot \textbf{W}_\mu\textbf{W}_h\textbf{V}$, where $\text{s}^+ = \sigma^+\left(\textbf{W}_m\norm{\textbf{W}_h\textbf{V}} + \textbf{b}_m\right)$ and $\norm{\cdot}$ is a row-wise norm. The idea is to show that $\textbf{W}_\mu\textbf{W}_h$ can extract three linearly independent vectors (and their negatives), and that the vector gate can construct the coefficients on these vectors so that they sum to $F(\textbf{V})$.

First, because $\textbf{v}_1, \textbf{v}_2, \textbf{v}_3$ are linearly independent, we can write $F(\textbf{V}) = \sum_{i=1}^3 c_i\textbf{v}_i$ for some coefficients $c_i$ dependent on $\textbf{V}$. That is, $c_i$ is a function $c_i(\textbf{V})$. Let $\displaystyle C = \max_{i=1,2,3} \max_{\textbf{V} \in \Omega^\nu} |c_i(\textbf{V})|$. Now define $c_i^+ = \max(c_i, 0)/C$ and $c_i^- = -\min(c_i, 0)/C$ such that $c_i = C(c_i^+ - c_i^-)$.  These functions $c_i^+$ and $c_i^-$ are bounded between 0 and 1, so we can then define $\tilde{c}^+_i = \sigma_g^{-1}(c_i^+)$ and correspondingly for $\tilde{c}^-_i$ These functions must be invariant under $R$, for otherwise $F$ would not be equivariant under $R$.

Next, let $\mathbf{W}_h$ be parameterized as described in \citet{jing2021learning}. Using the earlier result, there exists a form $\mathbf{w}_{i\pm}^T\sigma^+\left(\textbf{W}_{m,i\pm}\norm{\textbf{W}_h\textbf{V}} + \textbf{b}_m\right) \in \mathbb{R}$ parameterized by $\mathbf{w}_{i\pm}, \mathbf{W}_{m,i\pm}$ that  $\epsilon$-approximates each $\tilde{c}_i^\pm$. Then letting $\mathbf{W}_m = \begin{bmatrix} \mathbf{W}_{m,1+}^T & \mathbf{W}_{m,1-}^T & \cdots \end{bmatrix}$ and letting $\mathbf{W}_g$ have the $\textbf{w}_{i\pm}$ on the diagonals,  we see that the entries of $\mathbf{W}_g\text{s}^+$ will $\epsilon$-approximate the $\tilde{c}_i^\pm$. Now let $\textbf{b}_g = \textbf{0}$ such that the vector gate is $\sigma_g\left(\textbf{W}_g\text{s}^+\right) \in \mathbb{R}^6$ where $\text{s}^+$ is as defined above. Because $\sigma_g$ has less than unit derivative everywhere, we can see that the entries of $\sigma_g\left(\textbf{W}_g\text{s}^+\right)$ will $\epsilon$-approximate the bounded functions $c_i^\pm$.

Finally, let $\textbf{W}_\mu$ be parameterized such that $\textbf{W}_\mu\textbf{W}_h= C\begin{bmatrix} I_3 & \ldots \\ -I_3 & \ldots \end{bmatrix}$. (It is straightforward to do this given the $\textbf{W}_h$ described above.) Then  $\textbf{W}_\mu\textbf{W}_h\textbf{V} = C\begin{bmatrix} \textbf{v}_1, \textbf{v}_2, \textbf{v}_3, -\textbf{v}_1, -\textbf{v}_2, -\textbf{v}_3\end{bmatrix}^T \in \mathbb{R}^{6\times 3}$; call this matrix $\textbf{V}_{123}$. We can now see that the rows of $G_s(\textbf{V}) = \sigma_g(\textbf{W}_g\text{s}^+) \odot \textbf{V}_{123}$ will approximate $\pm Cc_i^\pm\textbf{v}_i$, and the error bound in each coordinate is $Cb\epsilon$, since no coordinate of $\textbf{v}_i$ is greater than $b$. Then the vector $f(\textbf{V}) = \textbf{1}^TG_s(\textbf{V})$ will approximate $\sum_{i=1}^3C(c_i^+ - c_i^-)\textbf{v}_i = \sum_{i=1}^3c_i\textbf{v}_i = F(\textbf{V})$, and the error bound in each coordinate is at most the sum of the six row error bounds.
\end{proof}

\end{document}